\documentclass{article}

\PassOptionsToPackage{numbers, sort}{natbib}
\usepackage[final]{neurips_2022}

\usepackage[utf8]{inputenc} %
\usepackage[T1]{fontenc}    %
\usepackage{hyperref}       %
\usepackage{url}            %
\usepackage{booktabs}       %
\usepackage{amsfonts}       %
\usepackage{nicefrac}       %
\usepackage{microtype}      %
\usepackage{xcolor}         %
\usepackage{autonum}
\usepackage{enumitem}
\usepackage{quoting}
\usepackage{wrapfig}
\usepackage[pdftex]{graphicx}

\usepackage{amsmath,amsfonts,bm, amsthm, mathrsfs, mathtools}

\newtheorem{proposition}{Proposition}

\newtheorem{problem}{Problem}

\numberwithin{theorem}{section}
\numberwithin{lemma}{section}
\numberwithin{proposition}{section}
\numberwithin{corollary}{section}

\def\eqref#1{equation~\ref{#1}}

\def\1{\bm{1}}

\def\rP{{\mathscr{P}}}

\def\ve{{\bm{e}}}

\def\vg{{\bm{g}}}

\def\vs{{\bm{s}}}

\def\vu{{\bm{u}}}
\def\vv{{\bm{v}}}
\def\vw{{\bm{w}}}
\def\vx{{\bm{x}}}
\def\vy{{\bm{y}}}
\def\vz{{\bm{z}}}

\DeclareMathAlphabet{\mathsfit}{\encodingdefault}{\sfdefault}{m}{sl}
\SetMathAlphabet{\mathsfit}{bold}{\encodingdefault}{\sfdefault}{bx}{n}

\def\gB{{\mathcal{B}}}

\def\gD{{\mathcal{D}}}

\def\gF{{\mathcal{F}}}

\def\gL{{\mathcal{L}}}

\def\gS{{\mathcal{S}}}

\def\gX{{\mathcal{X}}}
\def\gY{{\mathcal{Y}}}

\def\sR{{\mathbb{R}}}

\newcommand{\E}{\mathbb{E}}

\newcommand{\supp}{\text{supp}}

\DeclareMathOperator*{\argmin}{arg\,min}

\usepackage[ruled, lined, linesnumbered]{algorithm2e}

\usepackage{amssymb}
\SetKwInput{KwInput}{Input}                %
\SetKwInput{KwOutput}{Output}                %
\SetKwInput{KwReturn}{Return}
\SetKwInput{KwPara}{Parameter}

\usepackage{hyperref}
\hypersetup{
    colorlinks=true,
    citecolor =cyan,
    linkcolor=magenta,
    urlcolor=blue,
}
\makeatletter
\newcommand{\printfnsymbol}[1]{%
  \textsuperscript{\@fnsymbol{#1}}%
}
\makeatother

\title{Batch Active Learning from the Perspective of \\ Sparse Approximation}

\author{%
    \hspace{-0.7cm}Maohao Shen\thanks{Equal contribution. All authors completed the work at the University of Illinois Urbana-Champaign.} \\
    \hspace{-0.7cm}Massachusetts Institute of Technology\\
    \hspace{-0.7cm}\texttt{maohao@mit.edu}
    \And
    \hspace{-1cm}Bowen Jiang\printfnsymbol{1}\\
    \hspace{-1cm}University of Pennsylvania\\
    \hspace{-1cm}\texttt{bwjiang@seas.upenn.edu}
    \AND
    Jacky Y. Zhang\printfnsymbol{1}\\
    Stanford University\\
    \texttt{yiboz@stanford.edu}
    \And
    Oluwasanmi Koyejo \\
    Stanford University\\
    \texttt{sanmi@stanford.edu}
}

\begin{document}

\maketitle

\begin{abstract}
Active learning enables efficient model training by leveraging interactions between machine learning agents and human annotators. We study and propose a novel framework that formulates batch active learning from the sparse approximation's perspective. Our active learning method aims to find an informative subset from the unlabeled data pool such that the corresponding training loss function approximates its full data pool counterpart. We realize the framework as sparsity-constrained discontinuous optimization problems, which explicitly balance uncertainty and representation for large-scale applications and could be solved by greedy or proximal iterative hard thresholding algorithms. The proposed method can adapt to various settings, including both Bayesian and non-Bayesian neural networks. Numerical experiments show that our work achieves competitive performance across different settings with lower computational complexity.

\end{abstract}
\section{Introduction}\label{sec:intro}
Despite the promising results from deep neural networks, obtaining labels for a complex training datasets can still be challenging in practice. Often, this is because data annotation can be a time-consuming process that requires professional knowledge from human experts, e.g., in medicine~\citep{hoi2006batch, 9434098}. Different from traditional supervised learning techniques, \textit{Active Learning}~\citep{settles2009active}, as shown in Figure~\ref{fig: loop} curates interactions between machine learning agents and human annotators. Such a \textit{human-in-the-loop} learning can be used to mitigate the problem of scarce labeled data and enables efficient model training with limited annotation costs. Given a partially labeled dataset, active learning ideally selects data samples that are the best for learning. Specifically, it aims to iteratively query the most helpful data to ask an oracle (human annotator) to annotate. The queried data samples will be added back to the labeled data pool, and the model will be updated. This process is repeated until the model has achieved the desired performance. Intelligently identifying the most valuable data for annotation, also known as the query strategy, is the key problem in active learning.

A common approach is to employ the prediction uncertainty or data representation as the query metrics. For instance, \textit{uncertainty-based} approaches~\citep{settles2009active, tong2001support, gal2017deep, beluch2018power} work by querying samples with high uncertainty, but often results in selecting correlated and redundant data samples in each batch~\citep{kirsch2019batchbald, ducoffe2018adversarial}. On the other hand, \textit{representation-based} approaches~\citep{sener2017active, yang2019single} aim to select a subset of data that represents the whole unlabeled dataset, but tend to be computationally expensive and sensitive to batch sizes~\citep{ash2019deep, shui2020deep}. More recently, several \textit{hybrid} approaches~\citep{ash2019deep, shui2020deep, sinha2019variational} that try to consider both uncertainty and representation have shown advantages. Our work will take this hybrid view towards an active learning framework that \textit{balances the trade-off} between uncertainty and representation.
\begin{wrapfigure}{ro}{0.45\textwidth}
  \begin{center}
    \includegraphics[width=0.45\textwidth]{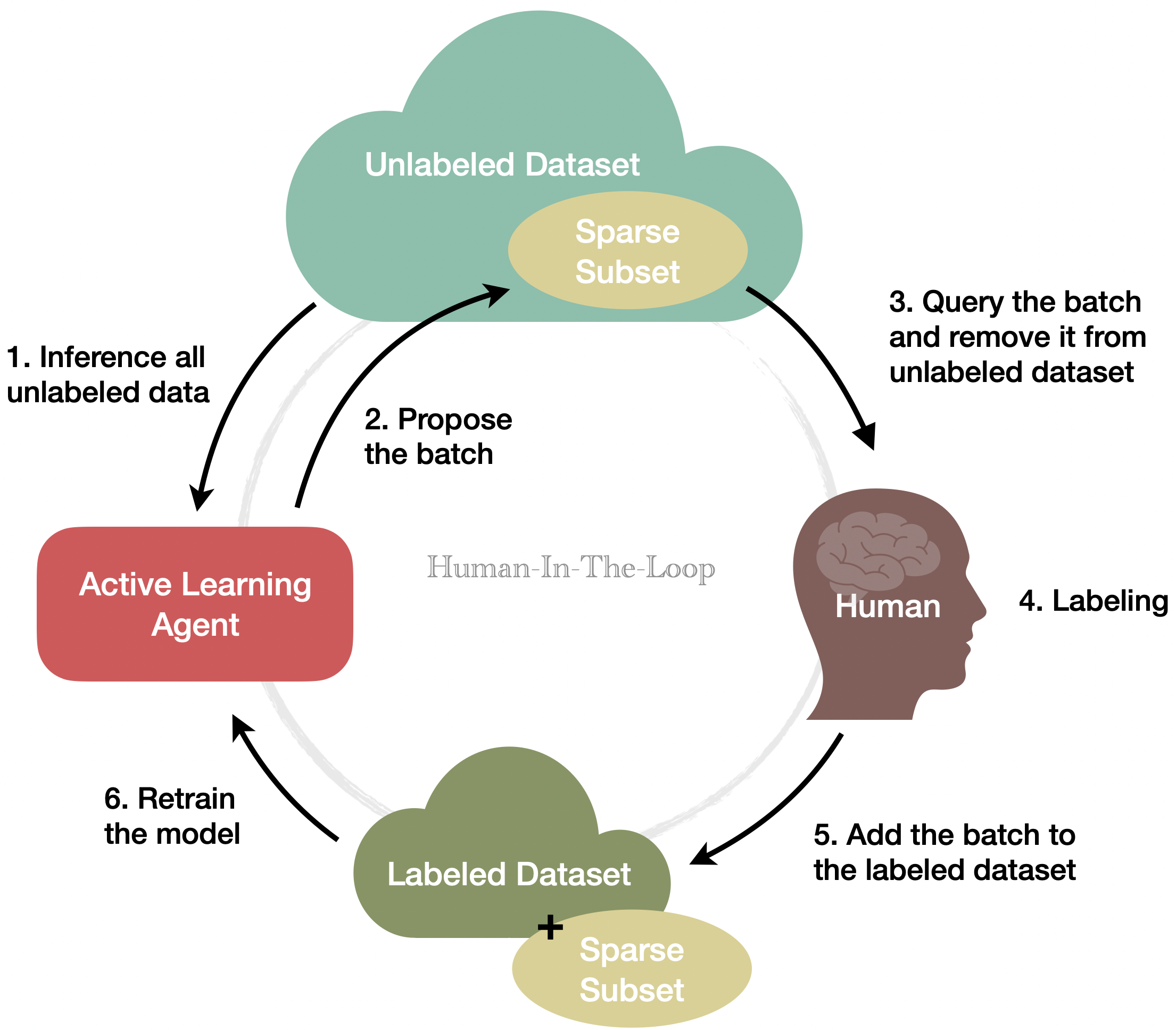}
  \end{center}
  \caption{Human-in-the-loop illustration.}
   \label{fig: loop}
  \vspace*{-2mm}
\end{wrapfigure}

Besides hybrid approaches, deep Bayesian active learning has also gained increasing attention. Several Bayesian approaches~\citep{gal2017deep, kirsch2019batchbald} leverage model uncertainty measurements~\citep{gal2015dropout, gal2016dropout} determined by Bayesian neural networks, while other works~\citep{pinsler2019bayesian} leverage recent progress in Bayesian Coreset problems~\citep{zhang2021bayesian, huggins2016coresets, campbell2019automated}. However, as most existing Bayesian approaches are explicitly designed for Bayesian neural networks, another goal of this paper is to propose a more \textit{general} approach that can adapt both Bayesian and non-Bayesian models.

For deep models, where updates after each batch can be computationally expensive, it is reasonable to query a large batch of data simultaneously to reduce model update frequency as this can be more efficient than sequential labeling. The batch selection setting is known as \textit{batch active learning}. However, from an optimization perspective, finding the best batch is NP-hard in general. Two common approaches for such combinatorial problems are the greedy and clustering approaches. Greedy algorithms select one data sample in sequence until the batch budget is exhausted~\citep{kirsch2019batchbald, biyik2019batch, chen2013near}. Here, specific conditions of the acquisition function, such as submodularity~\citep{nemhauser1978analysis}, are required to guarantee a good optimization result. Clustering algorithms regard cluster centers as their queried set~\citep{sener2017active, ash2019deep}, but they can be computationally expensive. To the best of our knowledge, except for \citet{pinsler2019bayesian} that focuses on Bayesian models, active learning has rarely been studied from a \textit{sparse approximation} perspective, despite the ubiquity of sparse approximation in signal processing for tasks such as dictionary learning~\citep{aharon2006k} and compressed sensing~\citep{donoho2006compressed}. Here we employ sparse approximation methods for batch active learning, leveraging their performance for discovering a sparse representation while avoiding redundancy. 

\textit{Our main contributions are summarized in the following:}
\vspace{-2mm}
\noindent
\begin{enumerate}
  \item We propose a flexible batch active learning framework from the perspective of sparse approximation, adaptable for both Bayesian and non-Bayesian settings.
  \item We realize this framework by deriving an upper bound to balance the trade-off between uncertainty and representation in a principled way.
  \item We approximate the loss functions that lead to a finite-dimensional, sparsity-constrained, and discontinuous optimization problem.
  \item We offer greedy and proximal iterative hard thresholding as two practical approaches for solving the optimization problem.
\end{enumerate}

The structure of this manuscript is as follows. We first formulate the active learning problem in Section~\ref{sec:prob-form}, and then realize the framework into a finite-dimensional discontinuous sparse optimization problem in Section~\ref{sec:finite}. To solve the resulting optimization problem, we demonstrate two optimization algorithms in Section~\ref{sec:method} with experimental results in Section~\ref{sec:exp}. Related work are discussed in Appendix~\ref{sec:related}. All proofs and training details are provided in Appendix~\ref{sec:proof} and \ref{sec:more_exp}. 
\section{Batch Active Learning as Sparse Approximation}
\label{sec:prob-form}
This section introduces the preliminaries and the conceptual formulation of batch active learning as a sparse approximation problem. 
\paragraph{Preliminaries} 
Vectors are denoted as bold lower case letters, \textit{e.g.}, $\vw\in \sR^n$. The $l_0$ pseudo-norm of a vector $\vw$ is denoted as $\|\vw\|_0$. We denote $\sR_+:=[0, +\infty)$. Distributions are denoted in script, \textit{e.g.}, $\rP$, and a random variable is denoted by tilde, \textit{e.g.}, $\tilde \vy \sim \rP$. We denote sets in calligraphy or in uppercase Greek alphabet (\textit{e.g.}, $\gD, \Theta$), and additionally we let $[n]:=\{1, 2, \dots, n\}$.  In supervised learning, given a labeled training dataset $\gD_l:=\{(\vx_i, \vy_i)\}_{i=1}^{n_l}$, where we denote their domain to be $\vx\in \gX$ and $\vy\in \gY$, the empirical goal is to minimize a loss function $L_l(\bm{\theta}):=\sum_{(\vx_i, \vy_i)\in \gD_l} \ell(\vx_i, \vy_i; \bm{\theta})$ formed by the training dataset, where $\bm{\theta}\in \Theta \subset \sR^m$ is the parameter of the model and $\ell$ is a loss function evaluated on individual pairs of data. Without loss of generality, we assume $\Theta\subset \sR^m$ is compact and $\ell(\vx, \vy; \cdot): \Theta \to \sR$ is in a normed space $(\gL(\Theta, \sR), \|\cdot\|_\dagger)$ for all $\vx, \vy$. We further assume the constant function $f:\Theta\to 1$ is included in $\gL(\Theta, \sR)$. The ``$\dagger$'' in the norm $\|\cdot\|_\dagger:\gL(\Theta, \sR) \to \sR_+$, representing its definition is a placeholder that will be discussed later.
 
\paragraph{Batch Active Learning} Besides the labeled dataset $\gD_l$, there is an unlabeled dataset $\gD_u:=\{\vx_j\}_{j=1}^{n_u}$ where the labels are unknown but could be acquired at a high cost through human labeling.  Combining two datasets, the ideal loss function to minimize w.r.t. $\bm{\theta}$ is
\begin{align}
    \textstyle \sum_{(\vx_i, \vy_i)\in \gD_l} \ell(\vx_i, \vy_i; \bm{\theta}) + \sum_{\vx_j \in \gD_u} \ell(\vx_j,\vy_j^\star;\bm{\theta}) , \label{eq:L-star}
\end{align}
where $\vy_j^\star$ is the unknown true label corresponding to the data $\vx_j$. Since acquiring true labels could be costly, we have to impose a budget $b$ ($b<n_u$) on the number of label acquisitions.  Therefore, the \textit{batch active learning} problem is to find a subset $\gS \subset \gD_u$ such that we can obtain a good model by optimizing the following loss function w.r.t. $\bm{\theta}$,
\begin{align}
    \textstyle\sum_{(\vx_i, \vy_i)\in \gD_l} \ell(\vx_i, \vy_i; \bm{\theta}) + \sum_{\vx_j \in \gS} \ell(\vx_j,\vy_j^\star;\bm{\theta}) , \quad \text{where } \ |\gS|=b. \label{eq:L-S-star}
\end{align}

\paragraph{Generalized Batch Active Learning} We start our method by generalizing the classical formulation (\eqref{eq:L-S-star}) by considering an importance weight for each unlabeled data. That is, we aim to find a sparse non-negative vector $\vw\in \sR^{n_u}_+$ such that we can obtain a good model by optimizing the following loss function w.r.t. $\bm{\theta}$:
\begin{align}
      \textstyle \sum_{(\vx_i, \vy_i)\in \gD_l} \ell(\vx_i, \vy_i; \bm{\theta}) + \sum_{\vx_j \in \gD_u} w_j \ell(\vx_j,\vy_j^\star;\bm{\theta}), \quad \text{where } \ \|\bm{w}\|_0=b. \label{eq:L-w-star}
\end{align}
A \textit{key question} now is---what is the criterion for a good $\bm{w}$? Comparing the ideal loss function (\eqref{eq:L-star}) and the sparse importance weighted loss (\eqref{eq:L-w-star}), the only difference is their unlabeled data loss functions. Therefore, a straight-forward informal criterion for a good importance weight $\vw$ is that the two unlabeled data loss functions are close to each other, \textit{i.e.},
\begin{align}
    \tilde{L}_{\vw}(\bm{\theta}):= \frac{1}{b}\sum_{\vx_j \in \gD_u} w_j \ell(\vx_j,\tilde{\vy}_j;\bm{\theta}) \quad \approx \quad \tilde{L}(\bm{\theta}):= \frac{1}{n_u}\sum_{\vx_j \in \gD_u} \ell(\vx_j,\tilde{\vy}_j;\bm{\theta}).  \label{eq:L-rv-approx}
\end{align}
Note that we change the notation $\vy_j^\star$ in equation~\ref{eq:L-w-star} to $\tilde{\vy}_j$ because true labels of unlabeled data are unknown. Luckily, we can have an estimator for the true labels, \textit{i.e.}, estimation based on the labeled data $p(\tilde \vy_j \mid \vx_j, \gD_l)$ or an approximation of it. Denote $\rP(\vx_j)$ as an estimated distribution, and therefore, $\tilde{\vy}_j \sim \rP(\vx_j)$. We are one step closer to evaluating the quality of a weighted selection. The next question is how to measure the difference between $\tilde L$ and $\tilde{L}_{\vw}$.

\paragraph{Difference Between Two Loss Functions} Given the two loss functions $\tilde{L}, \tilde{L}_{\vw}\in \gL(\Theta, \sR)$, where $\gL(\Theta, \sR)$ is equipped with the norm $\|\cdot\|_\dagger$, a straight-forward measurement of the difference between them is $\|\tilde{L}- \tilde{L}_{\vw}\|_\dagger$. However, observing that the optimization of a loss function is shift-invariant, the difference between two loss functions should also be shift-invariant. For example, for $\forall L\in \gL(\Theta, \sR)$ we have $\argmin_{\bm{\theta}\in \Theta} (L(\bm{\theta}) + c)  = \argmin_{\bm{\theta}\in \Theta} L(\bm{\theta})$ for $\forall c\in \sR$, implying that $L + c$ should be treated the same as $L$. Therefore, to account for the shift-invariance, we define $q:\gL(\Theta, \sR)\to \sR_+$ as
\begin{align}
    q(L):= \inf_{c\in \sR} \ \|L+c\|_\dagger, \qquad \forall L\in \gL(\Theta, \sR). \label{eq:q}
\end{align}
Note that \textit{ we abuse the notation} a bit, \textit{i.e.}, the $c$ in $L+c$ should be the constant function that maps every $\bm{\theta}\in \Theta$ to $c$. The above definition has some nice properties that make it a good difference measurement of two loss functions, as proved in proposition~\ref{prop:q} in the appendix. In particular, $q(\cdot)$ is a shift-invariant seminorm, i.e., it satisfies \textit{the triangle inequality} and $q(L+c)=q(L)$ for any constant $c$. Therefore, we can formulate the generalized batch active learning problem as the following sparse approximation problem.%

\begin{problem}[Batch Active Learning from the Perspective of Sparse Approximation] \label{prob:gbal} Given the shift-invariant seminorm $q$ induced by the norm $\|\cdot\|_\dagger$ (\eqref{eq:q}), and a label estimation distribution $\rP$, the generalized batch active learning problem (\eqref{eq:L-rv-approx}) is formally defined as
\begin{align}
    \argmin_{\vw\in \sR^{n_u}_+} \quad  \E_{\rP} [q( \tilde{L} - \tilde{L}_{\vw} )] \qquad \text{s.t.} \qquad  \|\vw\|_0 = b, \label{eq:approx-1}
\end{align}
where $\E_{\rP}$ stands for the expectation over $\tilde{\vy}_j\sim \rP(\vx_j)$ for $\forall j\in [n_u]$.
\end{problem}
Problem~\ref{prob:gbal} offers a general framework for batch active learning and can be applied with various settings, \textit{i.e.}, both the norm $\|\cdot\|_\dagger$ and the individual loss function $\ell$ can be chosen based on specific problems and applications. In the next section, we introduce two practical realizations of \eqref{eq:approx-1} for Bayesian and non-Bayesian active learning respectively.

\section{Sparse Approximation as Finite-dimensional Optimization}\label{sec:finite}
The approximation problem in~\eqref{eq:approx-1} is intractable. Therefore, we propose to transform it into a finite-dimensional sparse optimization problem. We address the issue regarding the sampling of $\E_\rP$, and discuss some concrete choices of $\rP$ and $\|\cdot\|_\dagger$ that lead to a solvable optimization function.

\paragraph{Addressing the Sampling Issue}
In \eqref{eq:approx-1}, the expectation  $\E_\rP$ is taken over the product space of $(\tilde{\vy}_1,\dots,\tilde{\vy}_{n_u})$ and each sample has to be remembered for future optimization, which can be intractable for large datasets. However, it has an upper bound where the complexity of the optimization is independent of the number of samples from $\rP$. First, by the triangle inequality 
\begin{align}
    \E_{\rP} [q( \tilde{L} - \tilde{L}_{\vw} )] &= \E_{\rP} [q( \tilde{L} - \E_{\rP}[\tilde{L}]+\E_{\rP}[\tilde{L}] - \E_{\rP}[\tilde{L}_{\vw}] + \E_{\rP}[\tilde{L}_{\vw}]  -\tilde{L}_{\vw} )]\\
    &\leq \underbrace{ \E_{\rP} [q( \tilde{L} - \E_{\rP}[\tilde{L}])] + \E_{\rP}[q( \tilde{L}_{\vw}-\E_{\rP}[\tilde{L}_{\vw}]  )]}_{(i): \text{ variance}}  + \underbrace{ q(\E_{\rP}[\tilde{L}] - \E_{\rP}[\tilde{L}_{\vw}])}_{(ii): \text{ approximation bias}}. \label{eq:ub}
\end{align}
We can see that it offers a trade-off between bias and variance, where the bias term is immediately tractable by expanding $\tilde{L},\tilde{L}_\vw$:
\begin{align}
    (ii)&=q(\E_{\rP}[\tfrac{1}{n_u}\textstyle\sum_{\vx_j \in \gD_u} \ell(\vx_j,\tilde{\vy}_j;\cdot)] - \E_{\rP}[\tfrac{1}{b}\textstyle\sum_{\vx_j \in \gD_u} w_j \ell(\vx_j,\tilde{\vy}_j;\cdot)])\\
    &=q( (\tfrac{1}{n_u}\textstyle\sum_{\vx_j \in \gD_u} \E_{\rP(\vx_j)}[\ell(\vx_j,\tilde{\vy}_j;\cdot)]) - (\tfrac{1}{b}\textstyle\sum_{\vx_j \in \gD_u} w_j \E_{\rP(\vx_j)}[\ell(\vx_j,\tilde{\vy}_j;\cdot)])). \label{eq:bias-1}
\end{align}
It remains to address the variance term \textit{(i)}. Recall that the more accurate $\rP$ is, the more accurate our approximation is. Given the decision $w_j>0$, if the label of $\vx_j$ is acquired, \textit{i.e.}, the oracle (human annotator) will offer us its true label $\vy_j^\star$, and the labeling distribution would be improved. That being said, the distribution of $\tilde{\vy}_j$ given $\vx_j$ and $w_j>0$ will be concentrated on its true label $\vy_j^\star$, \textit{i.e.},
\begin{align}
    \tilde{\vy}_j\sim \rP_\vw(\vx_j) := \begin{cases}
              \rP(\vx_j) \quad &\text{if $w_j=0$}\\
              \delta_{\vy_j^\star} \quad &\text{if $w_j>0$}
            \end{cases}, \qquad \text{where } \vw\in \sR^{n_u}_+\label{eq:Pw}
\end{align}
where $ \delta_{\vy_j^\star}$ denotes the distribution that $\tilde{\vy}_j$  can only be $\vy_j^\star$. However, the improved distribution $\rP_\vw$ is not known before the acquisition of the true labels $\vy^\star_j$ for $w_j>0$. Fortunately, although $\rP_\vw(\vx_j)$ is not known, it is known that the corresponding variance for $\tilde{\vy}_j$ would be zero no matter what its label is. Applying this trick, we show in the following proposition that the term \textit{(i)} with the improved label distribution $\rP_\vw$ has an upper bound that does not require to know the true labels.
\begin{proposition}\label{prop:var}
Let $\vw\in \sR^{n_u}_+$ and $\|\vw\|_0=b$, by replacing the $\rP$ by the improved estimation distribution $\rP_\vw$ (\eqref{eq:Pw}) into \textit{(i)} in \eqref{eq:ub}, we have 
\begin{align}
    \E_{\rP_\vw} [q( \tilde{L} - \E_{\rP_\vw}[\tilde{L}])] + \E_{\rP_\vw}[q( \tilde{L}_{\vw}-\E_{\rP_\vw}[\tilde{L}_{\vw}]  )] \leq \sum_{\vx_j \in \gD_u}  \bm{1}(w_j=0) \cdot \sigma_j,\label{eq:opt-var-1}
\end{align}
where $\sigma_j:=\frac{1}{n_u} \E_{  \rP(\vx_j)} [q(  \ell(\vx_j,\tilde{\vy}_j;\cdot) - \E_{ \rP(\vx_j)}[ \ell(\vx_j,\tilde{\vy}_j;\cdot)] )] $ is the individual variance, and $\bm{1}(\cdot)$ is the indicator function.
\end{proposition}
Therefore, combining \eqref{eq:bias-1} and \eqref{eq:opt-var-1}, we have a more tractable form of the sparse approximation, \textit{i.e.},
\begin{align}
    \argmin_{\vw\in \sR^{n_u}_+} \quad  q(\E_{\rP}[\tilde{L}] - \E_{\rP}[\tilde{L}_{\vw}]) + \sum_{\vx_j \in \gD_u} \bm{1}(w_j=0) \cdot \sigma_j \qquad \text{s.t.} \qquad  \|\vw\|_0 = b, \label{eq:approx-2}
\end{align}
Intuitively, such decomposition of \textit{bias} and \textit{variance} naturally provides metrics of \textit{uncertainty} and \textit{representation} for active learning, where the variance itself is a metric of uncertainty, meanwhile the bias term measures how well a subset of selected data can represent the whole unlabeled data. Now, it remains to specify the choice of $\|\cdot\|_\dagger$, \textit{i.e.}, the norm that induces $q$ (\eqref{eq:q}).

\paragraph{Formulation of the Finite-Dimensional Optimization}

We consider two concrete choices of the $\|\cdot\|_\dagger$ for Bayesian and non-Bayesian settings respectively.
\begin{enumerate}[leftmargin=*,itemsep=-0.5mm]
    \item In the Bayesian setting, we can easily sample $\bm{\theta}_i\sim \pi:=p(\bm{\theta}\mid \gD_l)$ from the posterior. Utilizing the posterior, we make the norm $\|\cdot\|_\dagger$ more concrete by considering the $L^2(\pi)$-norm, \textit{i.e.}, $\|L\|^2_{\pi} = \E_{\bm{\theta} \sim \pi} [L(\bm{\theta})^2]$. Accordingly,
\begin{align}
    q(L)^2=\inf_{c\in \sR} \|L+c\|^2_{\pi}=\inf_{c\in \sR} \E_{\bm{\theta} \sim \pi} [(L(\bm{\theta})+c)^2]= \E_{\bm{\theta} \sim \pi} [(L(\bm{\theta})-\E_{\bm{\theta} \sim \pi}[L(\bm{\theta})])^2]. \label{eq:L-pi}
\end{align}
The posterior $\pi$ tells us where and how to evaluate the ``magnitude'' of $L$.
  Noting that \eqref{eq:L-pi} is in the form of an expectation, we can draw $m$ samples $\bm{\theta}_i\sim\pi$ to approximate it. Denote $\vg:=\tfrac{1}{\sqrt{m}}[\dots, (L(\bm{\theta}_i)-\bar L),\dots]_{i=1\dots m}^\top\in \sR^m$ where $\bar L:= \tfrac{1}{m}\textstyle\sum_{i=1}^m L(\bm{\theta}_i)$. \eqref{eq:L-pi} becomes
    \begin{align}
        q(L)^2 \ \approx \  \frac{1}{m}\sum_{i=1}^m (L(\bm{\theta}_i)-\bar L)^2 = \|\vg\|^2_2,  \label{eq:norm-pi}
    \end{align}
    where $ \|\vg\|_2$ is simply the Euclidean norm of the $m$-dimensional vector $\vg$. \\
    
    \item In contrast, non-Bayesian neural networks provide a point estimate rather than a posterior distribution. Therefore, we evaluate the loss function in a local ``window'' based on the current model. We consider the $\|\cdot\|_\infty$-norm over a Euclidean ball $\gB_{r}(\bm{\theta}_0)$ of radius $r$ centered at the current model parameter $\bm{\theta}_0$, \textit{i.e.}, $\|L\|_\infty=\max_{\bm{\theta}\in \gB_{r}(\bm{\theta}_0)}|L(\bm{\theta})|$. Moreover, in the Euclidean ball we approximate $L(\bm{\theta}) \approx L(\bm{\theta}_0) + \nabla L(\bm{\theta}_0)^\top (\bm{\theta}-\bm{\theta}_0)$. Therefore, we have
      \begin{align}
        q(L) &= \inf_{c\in \sR} \|L+c\|_\infty=\inf_{c\in \sR} \max_{\bm{\theta}\in \gB_{r}(\bm{\theta}_0)}|L(\bm{\theta})+c|
        \approx 
        r\|\nabla L(\bm{\theta}_0)\|_2. \label{eq:norm-grad}
    \end{align}
    Note that $\|\nabla L(\bm{\theta}_0)\|_2$ is the Euclidean norm of the gradient vector $\nabla L(\bm{\theta}_0)\in \sR^m$.
\end{enumerate}

The label distribution $\rP(\vx_j)=p(\tilde{y}_j\mid \vx_j, \gD_l)$ can be directly estimated from posteriors' predictive distribution on Bayesian neural networks. For non-Bayesian models, one could utilize the calibrated model prediction~\citep{guo2017calibration} as the label distribution. Finally, plugging either of the two  approximations of $q(L)$ into \eqref{eq:approx-2}, squaring all of the terms for the ease of optimization,  and adding a regularization term, we can formulate the sparse approximation problem as the following finite-dimensional optimization problem. Detailed derivation of \eqref{eq:opt-2} is deferred to Appendix~\ref{sec:detail-prob-opt}.
\begin{problem}[Sparse Approximation as Finite-dimensional Optimization]\label{prob:opt}
The finite-dimensional optimization for generalized batch active learning is
\begin{align}
    \argmin_{\vw\in \sR^{n_u}_+}\quad  \|\vv - \Phi \vw\|^2_2 - \alpha \sum_{\vx_j \in \gD_u} \bm{1}(w_j>0) \cdot \sigma_j^2 + \beta  \|\vw-\bm{1}\|_2^2 \quad \text{s.t.} \quad \|\vw\|_0 = b. \label{eq:opt-2}
\end{align}
where $\alpha>0$ is to offer a trade-off between bias and variance, and $\beta$ term is a regularizer.
Moreover, $\vv:=\frac{1}{n_u} \sum_{j=1}^{n_u} \E_{\rP(\vx_j)}[ \vg_j(\tilde{\vy}_j)]$, $ \Phi:=\frac{1}{b}(\E_{\rP(\vx_1)}[ \vg_1(\tilde{\vy}_1)],\dots,\E_{\rP(\vx_{n_u})}[ \vg_{n_u}(\tilde{\vy}_{n_u})])$, and $\sigma_j=\frac{1}{n_u} \E_{\rP(\vx_j)}[\|\vg_j(\tilde{\vy}_j)- \E_{\rP(\vx_j)}[\vg_j(\tilde{\vy}_j)]\|_2]$, where
\begin{align}
    &\vg_j(\tilde{\vy}_j):=\begin{cases}
    [\dots,(\ell(\vx_j, \tilde{\vy}_j; \bm{\theta}_i)-\bar \ell),\dots]_{i=1\dots m}^\top, \quad \bar \ell := \frac{1}{m}\sum_{i=1}^m \ell(\vx_j, \tilde{\vy}_j;  \bm{\theta}_i) \quad &\text{ if use } (\ref{eq:norm-pi})\\
    \nabla \ell(\vx_j, \tilde{\vy}_j; \bm{\theta}_0) \quad   &\text{ if use } (\ref{eq:norm-grad}). 
    \end{cases}%
\end{align}
\end{problem}

In the next section, we propose two optimization algorithms for \eqref{eq:opt-2} by exploiting its unique properties. The overall procedure in practice is presented in Algorithm~\ref{alg:sabal} Appendix~\ref{sec:big-algorithm}.

\section{Optimization Algorithms}\label{sec:method}
This section focuses on optimizing Problem~\ref{prob:opt}. Rewrite {the objective function} of
equation \ref{eq:opt-2} as $ f(\bm{w}):=f_1(\bm{w})+ f_2(\bm{w})$, where
    $f_1(\bm{w}):= \|\vv - \Phi \vw\|^2_2 + \beta  \|\vw-\bm{1}\|_2^2$ and  $f_2(\bm{w}):= - \alpha \textstyle\sum_{\vx_j \in \gD_u} \bm{1}(w_j>0) \cdot \sigma_j^2$.
The optimization has two major difficulties, \textit{i.e.}, the nonconvex sparsity constraint $\|\vw\|_0=b$ and the discontinuous objective function $f_2$.  When it comes to sparsity-constrained optimization, there are two schemes that are widely considered --- greedy~\citep{nemhauser1978analysis, campbell2019automated} 
and proximal iterative hard thresholding (IHT)~\citep{zhang2021bayesian, khanna2018iht}. However, Problem~\ref{prob:opt} introduces the new difficulty other than the sparsity constraint, \textit{i.e.}, the discontinuous component $f_2$, which violate the assumptions of many of these methods which require the use of gradient.
Instead, we propose Algorithm~\ref{alg:gdy}\&\ref{alg:iht} specifically for Problem~\ref{prob:opt} under the two schemes respectively, while incorporating the discontinuity.

We introduce some notations used in this section. Given a vector $\vg$, we denote $[\vg]_+$ as $\vg$ with its negative elements set to $0$. For an index $j$, we denote $g_j$ or $(\vg)_j$ to be its $j^{th}$ element. For an index set $\gS$, we denote $[\vg]_{\gS}$ to be the vector where $([\vg]_{\gS})_j=(\vg)_j$ if $j\in \gS$ and $([\vg]_{\gS})_j=0$ if $j\notin \gS$. Moreover, we denote $\ve^j$ to be the unit vector where $(\ve^j)_j=1$ and $(\ve^j)_i=0$ for $\forall i\neq j$.

Although the two algorithms use different schemes, they share the same two sub-procedures: a line search and de-bias step (Algorithm~\ref{alg:ls} and~\ref{alg:db} in Appendix~\ref{sec:appendix-alg}), which significantly improve the optimization performance~\citep{zhang2021bayesian}. The line search sub-procedure optimally solves the problem $\argmin_{\mu\in \sR} \ f_1(\vw - \mu \vu)$, \textit{i.e.}, given a direction $\vu$, what is the best step size $\mu$ to move the $\vw$ along $\vu$. The de-bias sub-procedure adjusts a sparse $\vw$ in its own sparse support for a better solution.

\begin{wrapfigure}{R}{0.52\textwidth}
\vspace{-1.4em}
\begin{minipage}{0.52\textwidth}
\IncMargin{1.5em}
\begin{algorithm}[H]
\SetAlgoLined
\setcounter{AlgoLine}{0}
\Indm
\KwPara{sparsity $b$; step size $\tau$.}
\Indp

$\vw \leftarrow \bm{0}$; \ $\gS \leftarrow \emptyset$

\Repeat{$|\gS|=b$}{
  $j \leftarrow \underset{j\in [n_u] \backslash \gS}{\argmin} \ \tau (\nabla f_1(\vw))_j - \alpha \sigma_j^2$ 
  
  $\gS \leftarrow \gS \cup \{j\}$ \hfill {\small \textit{(update selection)}}
  
  $\mu \leftarrow$ LineSearch($\ve^j, \vw$) 
  
  $\vw\leftarrow$ De-bias($\vw-\mu \ve^j$) 
  
  $w_j\leftarrow 0$ for $\forall w_j<0$ \hfill {\small \textit{($\vw \in \sR^{n_u}_+$)}}

 }

\Indm
\KwReturn{$\vw$}
\Indp
 \caption{Greedy Algorithm}\label{alg:gdy}
\end{algorithm}
\DecMargin{1.5em}
\end{minipage}
\hfill
\begin{minipage}{0.52\textwidth}
\IncMargin{1.5em}
\begin{algorithm}[H]
\SetAlgoLined
\setcounter{AlgoLine}{0}
\Indm
\KwPara{sparsity $b$; number of iterations $T$.}
\Indp

$\vw \leftarrow \bm{0}$; \ $\vz\leftarrow \bm{0}$

\Repeat{$T$ iterations}{

$\vw'\leftarrow \vw$ \hfill {\small \textit{(save previous $\vw $)}}

$\mu \leftarrow$ LineSearch($\nabla f_1(\vz), \vz$)

$\vs\leftarrow \vz-\mu \nabla f_1(\vz)$ \hfill {\small \textit{(gradient descent)}}

$\vw \leftarrow \underset{\vw \in \sR^{n_u}_+, \|\vw\|_0\leq b}{\argmin} \tfrac{1}{2}\|\bm{w}-\bm{s}\|_2^2 + f_2(\bm{w})$ 

$\vw\leftarrow$ De-bias($\vw$)
  
$w_j\leftarrow 0$ for $\forall w_j<0$  \hfill {\small \textit{($\vw \in \sR^{n_u}_+$)}}

$\tau \leftarrow $ LineSearch($\vw - \vw', \vw$)

$\vz \leftarrow \vw - \tau (\vw-\vw')$  \hfill {\small \textit{(momentum)}}

}

\Indm
\KwReturn{$\vw$}
\Indp
 \caption{Proximal-IHT Algorithm}\label{alg:iht}
\end{algorithm}
\DecMargin{1.5em}
\end{minipage}
\vspace{-1em}
\end{wrapfigure}

\vspace{-0.5em}
\paragraph{Opt. Algorithm: Greedy} The core idea of the greedy approach is noted in line~3 Algorithm~\ref{alg:gdy}, where it chooses an index $j$ to move a step of size $\tau$ that minimizes the objective, \textit{i.e.},
$
    j\leftarrow \argmin_{j\in [n_u]\backslash\gS} \quad (f_1(\vw+\tau\ve^j)-f_1(\vw)) + (f_2(\vw+\tau\ve^j)-f_2(\vw)).
$
By approximating $f_1(\vw+\tau\ve^j)-f_1(\vw)$ by its first-order approximation $\langle \nabla f_1(\vw), \tau \ve^j\rangle$, and noting that $f_2(\vw+\tau\ve^j)-f_2(\vw)=-\alpha \sigma_j^2$, we have the greedy step (line~3) in Algorithm~\ref{alg:gdy}. After choosing the index $j$ to include, line~5 chooses an optimal step to move, followed by a de-bias step that further improves the solution in the current sparse support $\supp(\vw)$. 

\vspace{-0.5em}
\paragraph{Opt. Algorithm: Proximal iterative hard thresholding}
The core idea of the proximal IHT (Algorithm~\ref{alg:iht}) is noted in line~6, where it combines both the hard thresholding and the proximal operator. It minimizes the discontinuous $f_2$ in a neighbourhood of the solution $\vs$ obtained by minimizing $f_1$, while satisfying the constraints. As discussed in the section~\ref{sec:appendix-alg}, the inner optimization (line 6) can be done optimally by simply picking the top-$b$ elements from $n_u$ elements. 
After this core step, a de-bias step improves the solution $\vw$ within its sparse support, followed by a momentum step.

\vspace{-0.5em}
\paragraph{Complexity Analysis} %
We analyze time complexities of the proposed algorithms with respect to the number of data samples $n$ and the queried batch size $b$. Except for line~6 Algorithm~\ref{alg:iht}, all steps are of time complexity $O(n)$. The line~6 Algorithm~\ref{alg:iht} is finding the $b$ smallest elements, which can be done in $O(n\log(b))$. Therefore, the time complexity for the greedy algorithm is $O(nb)$, and the time complexity for the proximal IHT is $O(n\log(b))$. Compared to the time complexity $O(nb^2)$ of the state-of-the-art method BADGE~\citep{ash2019deep}, the two proposed algorithms can be much faster, especially with a large batch size $b$ in practice.

\section{Experiment results}\label{sec:exp}
We demonstrate the performance of our batch active learning framework on image classification tasks, and show its flexibility by first using Bayesian neural networks and then general convolutional neural networks. Besides, we show that our method has runtime advantages compared to other methods in the literature. Finally, we conduct an ablation study on the trade-offs between uncertainty and representation in Appendix~\ref{sec:exp-ablation} and describe training details in Appendix~\ref{sec:exp-2}.

Each experiment has a fixed training, validation, and testing set. The model is initially trained on small amounts of labeled data randomly selected from the training set and then the algorithm iteratively performs the data acquisition and annotation. The model is reinitialized and retrained at the beginning of each active learning iteration. After the model is well trained, its test accuracy is evaluated on the testing set as a performance measure. All experiments are repeated multiple times using 5 random seeds (3 for the small model LeNet-5~\citep{lecun2015lenet}), and the results are reported as mean and standard deviations. The performance of each iteration is shown in learning curve plots. To better visualize the overall performance of AL methods, We also measure the area under curve (AUC) scores of the learning curve of different AL methods across different datasets.

We implement proximal IHT and greedy as two optimizations for sparse approximation, denoted as \textit{Ours-IHT} and \textit{Ours-Greedy}. We compare with the baselines: \textbf{(1) Random}: A naive baseline that selects a batch uniformly at random. \textbf{(2) BALD}~\citep{houlsby2011bayesian}: An uncertainty-based Bayesian method that selects a batch of data with maximum mutual information between model parameters and predictions. \textbf{(3) Batch BALD}~\citep{kirsch2019batchbald}: A Bayesian method that extends BALD to estimate the mutual information between a joint of multiple data points and model parameters. \textbf{(4) Bayesian Coreset}~\citep{pinsler2019bayesian}: A Bayesian approach based on the Bayesian Coreset problem~\citep{huggins2016coresets,campbell2019automated}. \textbf{(5) Entropy}~\citep{wang2014new}: An uncertainty-based non-Bayesian method that selects a batch of data with maximum entropy of the model predictions $\mathbb{H}\left(\boldsymbol{y}_{i} \mid \boldsymbol{x}_{i}; \boldsymbol{\theta}\right)$. \textbf{(6) KCenter}~\citep{sener2017active}: A representation-based non-Bayesian method that reformulates the coreset selection as a KCenter problem in the feature embedding space. \textbf{(6) BADGE}~\citep{ash2019deep}: A hybrid non-Bayesian method that samples a diverse batch of data using the $k$-MEANS++ seeding algorithm.

\paragraph{Bayesian Active Learning}
We first perform Bayesian active learning with Bayesian neural networks on Fashion MNIST~\citep{lecun1998gradient}, CIFAR-10~\citep{krizhevsky2009learning}, and CIFAR-100~\citep{krizhevsky2009learning}. For a fair comparison, we keep the same experiment settings of~\citet{pinsler2019bayesian}, using a Bayesian neural network consisting of a ResNet-18~\citep{ he2016deep} feature extractor. The posterior inference is obtained by variational inference~\citep{wainwright2008graphical, blundell2015weight} at the last layer, and the model predictive posteriors $p(\tilde \vy_j \mid \vx_j, \gD_l)$ are estimated using 100 samples. Equation~\ref{eq:norm-pi} is used to solve the finite-dimensional optimization problem, because sampling from the posterior distribution in a Bayesian neural network will be efficient by leveraging the local reparameterization trick~\citep{kingma2015variational}. It can be seen in Table~\ref{table: AUC_Bayesian} and Figure~\ref{fig:bayesian} that both proximal IHT and greedy show some advantages on Fashion MNIST dataset. On CIFAR-10 and CIFAR-100, we find greedy performs better than proximal IHT while outperforming other baselines, including the Bayesian Coreset~\citep{pinsler2019bayesian}. We also notice that the performance of Batch BALD~\citep{kirsch2019batchbald} does not meet our expectations and even performs below BALD~\citep{houlsby2011bayesian}, except on CIFAR-100~\citep{krizhevsky2009learning} where it outperforms other methods. The possible explanation is that the original paper \citep{kirsch2019batchbald} uses tiny batch sizes, e.g., 5, 10, or 40, instead of larger batch sizes, like 1000 or 5000, which are more typical in batch active learning.

\begin{table}[h!]
\scriptsize
  \begin{center}
  \caption{AUC Score ($\pm$ std.) for different AL methods on Bayesian active learning. AUC measures the overall performance improvement across number of queries. Top two scores are \textbf{highlighted}.}
  \begin{tabular}{lllllllll}
    \textbf{Dataset} & \textbf{Bayesian Coreset} & \textbf{Ours-Greedy} & \textbf{Ours-IHT} & \textbf{BALD} & \textbf{Batch BALD} & \textbf{Random} \\
    \toprule
    Fashion MNIST & $89.53 \pm 0.25$ & $\mathbf{89.83 \pm 0.26}$ & $\mathbf{89.97 \pm 0.23}$ & $89.72 \pm 0.23$ & $88.85 \pm 0.22$ & $88.39 \pm 0.32$  \\
    CIFAR10 & $\mathbf{77.73 \pm 0.70}$ & $\mathbf{78.28 \pm 0.53}$ & $77.55 \pm 0.83$ & $77.61 \pm 0.59$ & $76.12 \pm 0.35$ & $76.36 \pm 0.59$  \\
    CIFAR100 & $42.40\pm 0.29$ & $\mathbf{42.53 \pm 0.34}$ & $42.30 \pm0.29 $ & $41.99 \pm 0.60$ & $\mathbf{42.63 \pm 0.40}$ & $42.10 \pm 0.24$
  \end{tabular}
  \label{table: AUC_Bayesian}
  \end{center}
\end{table}
\vspace{-2em}
\begin{figure}[htp]
  \centering
  \includegraphics[width=1\linewidth]{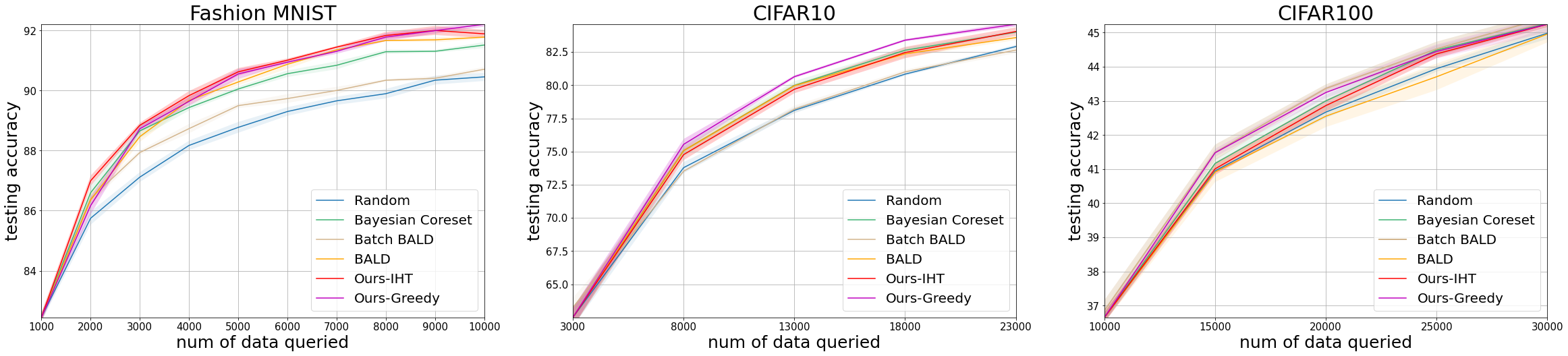}
  \caption{Active learning results on Bayesian models. Solid lines and shaded areas represent means and standard deviations of test accuracy over different seeds. Our method especially the greedy algorithm outperforms most baselines including the Bayesian Coreset~\citep{pinsler2019bayesian}.} %
  \label{fig:bayesian}
\end{figure}

\paragraph{General Active Learning}
We then implement our experiments on general convolutional neural networks, including LeNet-5~\citep{lecun2015lenet} and VGG-16~\citep{simonyan2014very} architectures without any Bayesian layers, using MNIST~\citep{lecun1998gradient}, SVHN~\citep{netzer2011reading}, and CIFAR-10~\citep{krizhevsky2009learning} datasets. We utilize calibrated prediction of current model with temperature scaling~\citep{guo2017calibration} to approximate the label distribution $p(\tilde \vy_j \mid \vx_j, \gD_l)$. Because the gradient of the last layer represents the full-model gradient~\citep{ash2019deep}, we can easily solve the optimization problem with gradient embedding as equation~\ref{eq:norm-grad}. We compare popular non-Bayesian baselines: Random, Entropy, KCenter, and BADGE. The results are shown in Table~\ref{table: AUC_general} and Figure~\ref{fig:non_bayesian}. In general, our method outperforms most baselines, achieving comparable performance to the strong baseline BADGE, but requires much less acquisition time, especially on large models. In addition, most methods perform similarly on CIFAR-10, and we conjecture that for CIFAR-10 each sample is informative enough; thus, random selection can achieve good enough performance.
\vspace{-1em}
\begin{table}[h!]
\scriptsize
  \begin{center}
  \caption{AUC Score ($\pm$ std.) for different AL methods on general active learning. AUC measures the overall performance improvement across number of queries. Top two scores are \textbf{highlighted}.}
  \begin{tabular}{lllllllll}
    \textbf{Dataset} & \textbf{BADGE} & \textbf{Ours-Greedy} & \textbf{Ours-IHT} & \textbf{KCenter} & \textbf{Entropy} & \textbf{Random} \\
    \toprule
    MNIST & $\mathbf{91.24 \pm 0.48}$ & $90.89 \pm 0.38$ & $\mathbf{91.07 \pm 0.45}$ & $89.57 \pm 1.02$ & $90.68 \pm 0.81$ & $86.48 \pm 1.11$ \\
    SVHN & $86.92 \pm 0.71$ & $\mathbf{87.23 \pm 0.47}$ & $86.84 \pm 0.57$ & $\mathbf{87.04 \pm 0.80}$ & $86.28 \pm 1.05$ & $85.52 \pm 0.51$ \\
    CIFAR10 &$\mathbf{68.20 \pm 0.56}$ & $\mathbf{68.01\pm 0.66}$ & $67.80 \pm 0.69$ & $67.98 \pm 0.63$ & $67.94 \pm 0.64$ & $67.05 \pm 0.59 $
  \end{tabular}
  \label{table: AUC_general}
  \end{center}
  \vspace{-3em}
\end{table}
\begin{figure}[htp]
  \centering
  \includegraphics[width=1\linewidth]{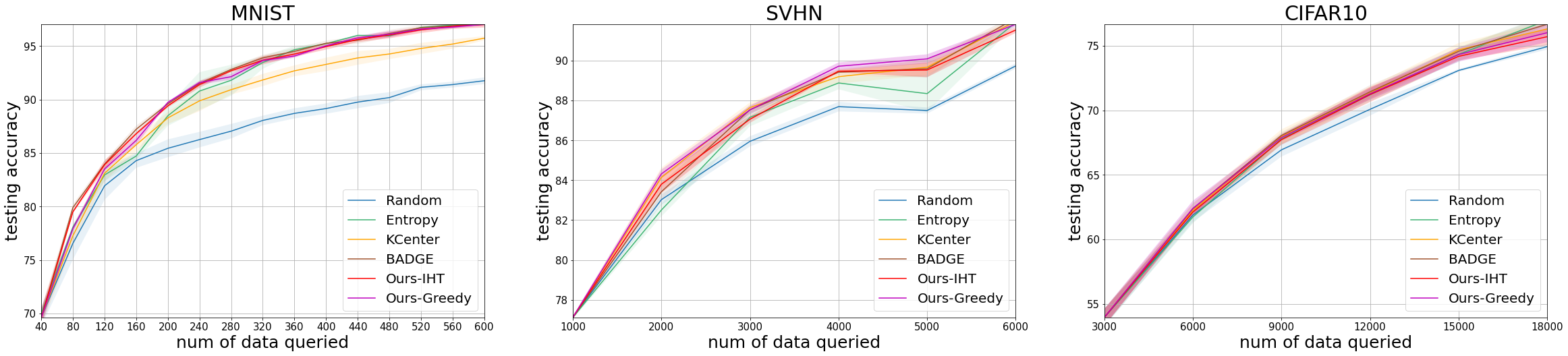}
  \caption{Active learning results on general (non-Bayesian) models. Solid lines and shaded areas represent means and standard deviations of test accuracy over different seeds.}
  \label{fig:non_bayesian}
\end{figure}
\vspace{-1em}
\paragraph{Run Time Comparison}
To show that our methods have competitive performance with less acquisition time, we compare the empirical runtime complexity with other baselines in non-Bayesian active learning. Here, we consider the acquisition time of the first query, where the unlabeled data pool has the largest size compared with later query iterations. Large models and datasets (SVHN and CIFAR-10 on VGG-16) are used to better illustrate the runtime complexity. Results are shown in Table~\ref{table: time}. It can be seen that our methods require much less runtime than BADGE especially when the queried batch is large, and even less than KCenter in most cases.
\begin{table}[h!]
\scriptsize
  \begin{center}
  \caption{first query's acquisition time ($\pm$ std.) of different AL methods on two large datasets. Our methods show big runtime advantages over other non-trivial methods.}
  \begin{tabular}{llllll}
    \textbf{Dataset} & \textbf{Method} & \textbf{Time (unit:s)} & \textbf{Dataset} & \textbf{Method} & \textbf{Time (unit:s)} \\
    \toprule
    SVHN & BADGE & $732.18 \pm 26.29$ & CIFAR10 & BADGE & \hspace{-1.5em}$1207.19 \pm 121.09$ \\
     & Ours-Greedy & $201.65 \pm 3.54$ & & Ours-Greedy & $333.67 \pm 3.53$ \\
     & Ours-IHT & $211.04 \pm 10.65$ & & Ours-IHT & $174.28 \pm 3.27$ \\
     & KCenter & $309.99 \pm 0.81$ & & KCenter & $258.29 \pm 1.75$ \\
     & Entropy & $\ \ 16.46 \pm 0.29$ & & Entropy & $\ \ 11.76 \pm 0.03$ \\
     & Random & $\ \ \ \ 0.81 \pm 0.02$ & & Random & $\ \ \ \ 1.42 \pm 0.02$
  \end{tabular}
  \label{table: time}
  \end{center}
\end{table}
\vspace{-0.5cm}

\section{Concluding Remarks}\label{sec:con}  
This work introduces a novel framework that generalizes batch active learning, a human-machine interactive learning mechanism, as a sparse approximation problem with a principled balance between representation and uncertainty. The central intuition is finding a \textit{weighted subset} from the unlabeled data pool whose corresponding training loss approximates the \textit{full-set} loss function in a function space. Specifically, we realize the framework as a finite-dimensional optimization problem, efficiently solvable by the greedy or proximal IHT algorithms, with the flexibility to adapt to both Bayesian and non-Bayesian settings. Experiments demonstrate strong performance with lower time complexity. 

In the future, we will consider an in-depth theoretical analysis on how to strike the best balance between the variance and bias controlled by hyperparameters. Besides, since statistical and analytic approximations seem to be inevitable for solving the problem in practice, we would also like to explore more realization and optimization techniques with more theoretical inspiration and, if possible, find precise bounds for these approximations.
\newpage
\section*{Acknowledgement}
Koyejo acknowledges partial funding from a Sloan Fellowship. This work was also funded in part by NSF 2046795, 1909577, 1934986 and NIFA award 2020-67021-32799.

\bibliography{refs}

\clearpage
\appendix
\begin{center}
\LARGE
\textbf{Batch Active Learning from the Perspective of\\Sparse Approximation}

\textbf{Appendix}
\end{center}

\section{The Overall Procedure}\label{sec:big-algorithm}
\IncMargin{1.5em}
\begin{algorithm}[htp]
\SetAlgoLined
\Indm  
\KwInput{Intial parameters $\boldsymbol{\theta}$, initial unlabeled pool $\mathcal{D}_{u}$, initial labeled pool $\mathcal{D}_{l} = \varnothing$, initial number of samples $b_{0}$, query batch size $b$, number of iterations $T$.}
\Indp
Query a random batch $\mathcal{S}_{0}$ of $b_{0}$ data from $\mathcal{D}_{l}$, update $\mathcal{D}_{u} \leftarrow \mathcal{D}_{u} \backslash \mathcal{S}_{0}$ and $\mathcal{D}_{l} \leftarrow \mathcal{D}_{l} \cup \mathcal{S}_{0}$.

Train the model using $\mathcal{S}_{0}$.

\For{$t = 1,2, \ldots, T$}{
    For each data $\boldsymbol{x}_j \in \mathcal{D}_{u}$, estimate its label distribution $\tilde{\vy}_j \sim \rP(\vx_j)$.
    
    Compute vector $\vg_j$ by sampling or gradient embedding using~\ref{eq:g-proj}.
    
    Compute $\boldsymbol{v}$, $\Phi$, and $\sigma_j$ for each $\boldsymbol{x}_j \in \mathcal{D}_{u}$ and form equation~\ref{eq:opt-2}. %

    Find sparse weight $\vw$ s.t. $\|\vw\|_{0}=b$ as specified in section~\ref{sec:method}.
    
    Select a batch of data $\mathcal{S}_{t} = \{ \boldsymbol{x}_j\in \mathcal{D}_{u} \mid \vw_j > 0\} $ and query their labels.
    
    Update $\mathcal{D}_{u} \leftarrow \mathcal{D}_{u} \backslash \mathcal{S}_{t}$ and $\mathcal{D}_{l} \leftarrow \mathcal{D}_{l} \cup \mathcal{S}_{t}$.
    
    Reinitialize and retrain the model using updated $\mathcal{D}_{l}$, update model parameters $\boldsymbol{\theta}$
}
\Indm  
\KwReturn{Final model parameters $\boldsymbol{\theta}$.}
\Indp
 \caption{Batch Active Learning from the Perspective of Sparse Approximation}\label{alg:sabal}
\end{algorithm}
\DecMargin{1.5em}

\section{Related Work}\label{sec:related}
Active learning has been widely studied by the machine learning community. As most classic approaches have already been discussed in a detail in~\citet{settles2009active, dasgupta2011two, hanneke2014theory}, we will briefly review some recent works in deep active learning. 

Existing query strategies can mainly be categorized as uncertainty-based and representation-based. Uncertainty-based approaches look for data samples the model is mostly uncertain about. Meanwhile, under the Bayesian setting, several recent works leverage the Bayesian neural network to well measure the model uncertainty. \citet{gal2015dropout, gal2016dropout} proves the Monte-Carlo dropout (MC Dropout) as an approximation of performing Bayesian inference, and enables efficient uncertainty estimations in neural networks. \citet{gal2017deep} utilizes MC Dropout for approximating posterior distributions and adapts \citet{houlsby2011bayesian} as their uncertainty based acquisition function, and similarly, \citet{kirsch2019batchbald} proposes a batch-mode approach based on \citet{gal2017deep} and shows some improvements through a more accurate measurement of mutual information between the data batch and model parameters.
While MC Dropout becomes prevalent for uncertainty estimation, \citet{beluch2018power} shows ensemble-based methods lead to better performance because of more calibrated uncertainty estimation, and another recent work~\citet{hemmer2020deal} also proposes a new uncertainty estimation method by replacing the softmax output of a neural network with the parameter of Dirichlet density. Other non-Bayesian approaches sometimes combine uncertainty estimation with other metrics: \citet{li2013adaptive} combines an information density measure to maximize the mutual information between selected samples and remaining unlabeled samples under the Gaussian Process setting.~\citet{wang2016cost} selects data based on several classic uncertainty metrics and incorporate a cost-efficient strategy by pseudo labeling the confident samples. 

Representation-based approaches attempt to query diverse data samples that could best represent the overall unlabeled dataset. A recent work proposed by~\citet{sener2017active} defines the active learning as a core-set selection problem. They derive an upper bound for the core-set loss and construct representative batches by solving a k-Center problem in the feature space. In~\citet{geifman2017deep}, the authors also explore the deep active learning with core-sets, but build the core-sets in the farthest-first compression scheme. 

In hybrid active learning literature, one of the state-of-art methods, BADGE~\citep{ash2019deep}, captures uncertainty through the lens of gradients, and samples diverse batches on the gradient embedding by the $k$-MEANS++ seeding algorithm. However, one of the downsides of BADGE is the high run-time complexity, as data acquisition speed is crucial in practice. \citep{sinha2019variational} train a Variational Autoencoder and a discriminator in an adversarial fashion. The discriminator predicts a sample as unlabeled based on its likelihood of representativeness, and a batch of samples with the lowest confidence will be queried. However, their adversarial method is difficult to apply to general and Bayesian neural networks. Our proposed method explicitly balances the trade-offs between uncertainty and representation by bias and variance decomposition.

Coreset selection is a common high-level idea in active learning, and methods vary in how one characterizes the closeness of a coreset to the full-set. \citep{sener2017active} characterizes the closeness as how much a coreset covers the full-set in the Euclidean distance in a feature space. It derives an upper bound for the coreset loss based on the Lipschitz continuity and transforms the original problem into a KCenter problem. However, their method relies on good feature representation, which is not always guaranteed in practice. \citep{pinsler2019bayesian} is mainly based on existing Bayesian inference literature, especially the Bayesian Coreset problem~\citep{campbell2019automated}. They characterize the closeness as how much the core-set log-posterior approximates the full-set log-posterior, with the log-posterior directly derived from the Bayes' rule. However, their problem formulation relies on the Bayesian setting and Bayesian models, and conducting posterior inference is non-trivial for non-Bayesian models. In contrast, our method characterizes the closeness in a more general sense, i.e., through a semi-norm function directly on the difference between the coreset loss function and the full-set loss function.

\section{Proofs}\label{sec:proof}
\begin{proposition}\label{prop:q}
$q:\gL(\Theta, \sR)\to \sR_+$ defined in (\ref{eq:q}) is a shift-invariant seminorm satisfying the following properties:
\begin{enumerate}
    \item $q(L_1+L_2)\leq q(L_1)+q(L_2)$ for $\forall L_1, L_2\in \gL(\Theta, \sR)$; \hfill  (triangle inequality)
    \item $q(cL)=|c|q(L)$ for $\forall L\in \gL(\Theta, \sR), \forall c\in \sR$; \hfill (absolute homogeneity)
    \item $q(L+c)=q(L)$ for $\forall L\in \gL(\Theta, \sR), \forall c\in \sR$; \hfill (shift-invariance)
    \item $q(L)=0$ if and only if $L$ maps every $\bm{\theta}\in \Theta$ to a constant.
\end{enumerate}
In other words, $q$ defines a norm in the space of shift-equivalence classes of loss functions.
\end{proposition}
\begin{proof}
Recall that 
\begin{align}
    q(L):= \inf_{c\in \sR} \ \|L+c\|_\dagger, \qquad \forall L\in \gL(\Theta, \sR). 
\end{align}
We prove the four properties respectively in the following.
\begin{enumerate}
    \item The triangle inequality is inherited from the sub-additivity  of  the norm $\|\cdot\|_\dagger$. For $\forall L\in \gL(\Theta, \sR)$, we have
    \begin{align}
        q(L_1+L_2)&=\inf_{c\in \sR} \ \|L_1+L_2+c\|_\dagger= \inf_{c_1, c_2\in \sR} \ \|L_1+L_2+c_1+c_2\|_\dagger\\
        &\leq \inf_{c_1, c_2\in \sR} \|L_1+c_1\|_\dagger+\|L_2+c_2\|_\dagger\\
        &=(\inf_{c\in \sR} \|L_1+c\|_\dagger) + (\inf_{c\in \sR} \|L_2+c\|_\dagger)\\
        &=q(L_1)+q(L_2).
    \end{align}
    \item The absolute homogeneity is also inherited from the absolute homogeneity of  the norm $\|\cdot\|_\dagger$. The case for $c=0$ is obvious, and for $c\neq 0$ we have
    \begin{align}
        q(cL)=|c|q(L)&=\inf_{c_1\in \sR} \ \|cL+c_1\|_\dagger=\inf_{c_1\in \sR} \ |c|\cdot\|L+c_1/c\|_\dagger\\
        &=\inf_{c_2\in \sR} \ |c|\cdot\|L+c_2\|_\dagger =|c|q(L).
    \end{align}
    \item By the definition of $q(\cdot)$, we have the shift-invariance of $q(\cdot)$.
    \item The ``if'' part can be proved by definition, \textit{i.e.}, $q(c)=\inf_{c_1\in \sR} \ \|c_1+c\|_\dagger=\|0\|_\dagger=0$. 
    
    For the ``only if'' part, we need to be more rigorous by defining $f_c$ to be the function that maps $\Theta$ to $c\in \sR$. We further define $\gF:=\{f_c\mid c\in \sR\}\subset \gL(\Theta, \sR)$, and we can see $(\gF, \|\cdot\|_\dagger)$ is a one-dimensional normed space. Letting $L\in \gL(\Theta, \sR)$ and $q(L)=0$, we have 
    \begin{align}
        \inf_{f_{c_\epsilon}\in \gF} \ \|L+f_c\|_\dagger = 0.
    \end{align}
    Therefore, for $\forall \epsilon>0$, $\exists c_\epsilon\in \sR$ such that
    \begin{align}
        &\|L+f_{c_\epsilon}\|_\dagger \leq \epsilon\\
         \implies &\|f_{c_\epsilon}\|_\dagger = \|L+f_{c_\epsilon}-L\|_\dagger \leq \epsilon+\|L\|_\dagger.
    \end{align}
    That being said, for $0<\epsilon<1$, we have $\|f_{c_\epsilon}\|_\dagger\leq 1+\|L\|_\dagger$. Denote $\gB=\{f_c\in \gF \mid \|f_c\|_\dagger\leq 1+\|L\|_\dagger \}$, and we can see $\gB$ is a closed ball in $\gF$. As $\gF$ is one-dimensional, by Riesz's lemma we have $\gB$ compact. 
    
    As $\lim_{\epsilon\to 0} \|L+f_{c_\epsilon}\|_\dagger=0$, \textit{i.e.}, $f_{c_\epsilon}\to L$, by the compactness of $\gB$ we have $L\in \gB$. Therefore, $L$ is also a constant function. Note that this conclusion does not require $\gL(\Theta, \sR)$ to be complete.

\end{enumerate}
\end{proof}

\begin{proposition}[Proposition~\ref{prop:var} Restated]
As $\vw\in \sR^{n_u}_+$ and $\|\vw\|_0=b$, by replacing the $\rP$ by the improved estimation distribution $\rP_\vw$ (\eqref{eq:Pw})  into \textit{(i)} in \eqref{eq:ub}, we have 
\begin{align}
    \E_{\rP_\vw} [q( \tilde{L} - \E_{\rP_\vw}[\tilde{L}])] + \E_{\rP_\vw}[q( \tilde{L}_{\vw}-\E_{\rP_\vw}[\tilde{L}_{\vw}]  )] \leq \sum_{\vx_j \in \gD_u}  \bm{1}(w_j=0) \cdot \sigma_j,\label{eq:opt-var-1-appendix}
\end{align}
where  $\sigma_j:=\frac{1}{n_u} \E_{  \rP(\vx_j)} [q(  \ell(\vx_j,\tilde{\vy}_j;\cdot) - \E_{ \rP(\vx_j)}[ \ell(\vx_j,\tilde{\vy}_j;\cdot)] )] $  is the individual variance, and $\bm{1}(\cdot)$ is the indicator function. 
\end{proposition}
\begin{proof}
Recall that
\begin{align}
    &\tilde{L}_{\vw}(\bm{\theta}):= \frac{1}{b}\sum_{\vx_j \in \gD_u} w_j \ell(\vx_j,\tilde{\vy}_j;\bm{\theta}), \quad  \quad &&\tilde{L}(\bm{\theta}):= \frac{1}{n_u}\sum_{\vx_j \in \gD_u} \ell(\vx_j,\tilde{\vy}_j;\bm{\theta}),\\
    & \tilde{\vy}_j\sim \rP_\vw(\vx_j) := \begin{cases}
              \rP(\vx_j) \quad &\text{if $w_j=0$}\\
              \delta_{\vy_j^\star} \quad &\text{if $w_j>0$}
            \end{cases},
            && \vw\in \sR^{n_u}_+,
\end{align}
where $ \delta_{\vy_j^\star}$ denotes the distribution that $\tilde{\vy}_j$  can only be $\vy_j^\star$. Therefore, by the definition of $\rP_\vw$, we have
\begin{align}
    \E_{\rP_\vw(\vx_j)} \left[q\left(  \ell(\vx_j,\tilde{\vy}_j;\cdot) - \E_{\rP_\vw(\vx_j)}[\ell(\vx_j,\tilde{\vy}_j;\cdot)]\right)\right]=0, \qquad \text{if $w_j>0$}.
\end{align}

Plugging the above definitions into $ \E_{\rP_\vw}[q( \tilde{L}_{\vw}-\E_{\rP_\vw}[\tilde{L}_{\vw}]  )]$ we have
\begin{align}
     \E_{\rP_\vw}[q( \tilde{L}_{\vw}-\E_{\rP_\vw}[\tilde{L}_{\vw}]  )]=\E_{\rP_\vw} \left[q\left( \frac{1}{n_u}\sum_{\vx_j \in \gD_u} w_j\left( \ell(\vx_j,\tilde{\vy}_j;\cdot) - \E_{\rP_\vw(\vx_j)}[\ell(\vx_j,\tilde{\vy}_j;\cdot)]\right)\right)\right] &\\
     =\E_{\rP_\vw} \left[q\left( \frac{1}{n_u}\sum_{\vx_j \in \gD_u} \bm{1}(w_j>0)  w_j\left( \ell(\vx_j,\tilde{\vy}_j;\cdot) - \E_{\rP_\vw(\vx_j)}[\ell(\vx_j,\tilde{\vy}_j;\cdot)]\right)\right)\right] 
     = 0.& \label{eq-app:1}
\end{align}
Therefore, we only need to care about the 
$\E_{\rP_\vw} [q( \tilde{L} - \E_{\rP_\vw}[\tilde{L}])]$. 
\begin{align}
    \E_{\rP_\vw} [q( \tilde{L} - \E_{\rP_\vw}[\tilde{L}])]&=\E_{\rP_\vw} \left[q\left( \frac{1}{n_u}\sum_{\vx_j \in \gD_u} \ell(\vx_j,\tilde{\vy}_j;\cdot) - \E_{\rP_\vw(\vx_j)}[\ell(\vx_j,\tilde{\vy}_j;\cdot)]\right)\right] \\
    &\leq \sum_{\vx_j \in \gD_u} \frac{1}{n_u}\E_{\rP_\vw(\vx_j)} \left[q\left(  \ell(\vx_j,\tilde{\vy}_j;\cdot) - \E_{\rP_\vw(\vx_j)}[\ell(\vx_j,\tilde{\vy}_j;\cdot)]\right)\right]\\
    &= \sum_{\vx_j \in \gD_u} \bm{1}(w_j=0) \frac{1}{n_u}\E_{\rP} \left[q\left(  \ell(\vx_j,\tilde{\vy}_j;\cdot) - \E_{\rP(\vx_j)}[\ell(\vx_j,\tilde{\vy}_j;\cdot)]\right)\right]\\
    &=\sum_{\vx_j \in \gD_u}  \bm{1}(w_j=0) \cdot \sigma_j, \label{eq-app:2}
\end{align}
where the inequality is by the triangle inequality and the absolute homogeneity of $q(\cdot)$ (Proposition~\ref{prop:q}). Combining \eqref{eq-app:1} and \eqref{eq-app:2}, we have the proposition proved.
\end{proof}

\section{Detailed Derivation of Problem~\ref{prob:opt}}\label{sec:detail-prob-opt}
Recall the two approximation of $q(L)$ (\eqref{eq:norm-pi} and \eqref{eq:norm-grad}), plugging either of them into \eqref{eq:approx-2}, and squaring all of the terms for the ease of optimization,  we can formulate the sparse approximation problem as the following finite-dimensional optimization problem, where $\alpha>0$ is added to offer a trade-off between bias and variance.
\begin{align}
     \argmin_{\vw\in \sR^{n_u}_+}\quad  \|\vv - \Phi \vw\|^2_2 + \alpha \sum_{\vx_j \in \gD_u} \bm{1}(w_j=0) \cdot \sigma_j^2 \quad \text{s.t.} \quad \|\vw\|_0 = b, \label{eq:opt-1}
\end{align}
where we denote $\vv\in \sR^m$, $\Phi \in \sR^{m \times n_u}$ and $\sigma_j$ as
\begin{align}
    &\vv:=\frac{1}{n_u} \sum_{j=1}^{n_u} \E_{\rP(\vx_j)}[ \vg_j(\tilde{\vy}_j)],\qquad \Phi:=\frac{1}{b}(\E_{\rP(\vx_1)}[ \vg_1(\tilde{\vy}_1)],\dots,\E_{\rP(\vx_{n_u})}[ \vg_{n_u}(\tilde{\vy}_{n_u})]), \label{eq:opt-vphi}\\
    &\sigma_j=\frac{1}{n_u} \E_{\rP(\vx_j)}[\|\vg_j(\tilde{\vy}_j)- \E_{\rP(\vx_j)}[\vg_j(\tilde{\vy}_j)]\|_2], \label{eq:opt-sig}\\
    &\vg_j(\tilde{\vy}_j):=\begin{cases}
    [\dots,(\ell(\vx_j, \tilde{\vy}_j; \bm{\theta}_i)-\bar \ell),\dots]_{i=1\dots m}^\top, \quad \bar \ell := \frac{1}{m}\sum_{i=1}^m \ell(\vx_j, \tilde{\vy}_j;  \bm{\theta}_i) \quad &\text{ if use } (\ref{eq:norm-pi})\\
    \nabla \ell(\vx_j, \tilde{\vy}_j; \bm{\theta}_0) \quad   &\text{ if use } (\ref{eq:norm-grad}). \label{eq:g-proj}
    \end{cases}
\end{align}
In practice, it is often the case that the number of parameters is less than the number of samples, \textit{i.e.}, $m<n_u$, even for over-parameterized neural networks where the gradient of the last layer is commonly used to represent the full-model gradient~\citep{katharopoulos2018not, ash2019deep}. Therefore, if the batch size is big, \textit{i.e.}, $b>m$, the approximation bias $ \|\vv - \Phi \vw\|^2_2$ may be under-determined with infinitely many $\vw$ to make $\vv = \Phi \vw$, and the optimization (\eqref{eq:opt-1}) may be ''overfitted''. To make our method more stable, we include a $\ell_2$ regularizer $\beta \|\vw-\bm{1}\|_2^2$ with $\beta>0$. Finally, since $w_j\geq 0$, minimizing $\alpha \sum_{\vx_j \in \gD_u} \bm{1}(w_j=0) \cdot \sigma_j^2 $ is equivalent to minimizing $-\alpha \sum_{\vx_j \in \gD_u} \bm{1}(w_j>0) \cdot \sigma_j^2 $. Consequently, we have the following optimization problem (Problem~\ref{prob:opt}). 
\begin{align}
    \argmin_{\vw\in \sR^{n_u}_+}\quad  \|\vv - \Phi \vw\|^2_2 - \alpha \sum_{\vx_j \in \gD_u} \bm{1}(w_j>0) \cdot \sigma_j^2 + \beta  \|\vw-\bm{1}\|_2^2 \quad \text{s.t.} \quad \|\vw\|_0 = b. \label{eq:opt-2-1}
\end{align}

\section{Omitted Algorithms}\label{sec:appendix-alg}
In this section we present the two sub-procedures, \textit{i.e.}, line search and de-bias, shared by two main optimization algorithms (Algorithm~\ref{alg:gdy}\&\ref{alg:iht}), as well as how the optimization (line~6) in Algorithm~\ref{alg:iht} is solved optimally.

The line search sub-procedure (Algorithm~\ref{alg:ls}) optimally solve the problem of $\argmin_{\mu\in \sR} \ f_1(\vw - \mu \vu)$, \textit{i.e.}, given a direction $\vu$ what is the best step size to move the $\vw$ along $\vu$. The de-bias sub-procedure (Algorithm~\ref{alg:db}) adjusts a sparse $\vw$ in its own sparse support for a better solution.

\begin{minipage}{0.48\textwidth}
\IncMargin{1.5em}
\begin{algorithm}[H]
\SetAlgoLined
\Indm 
\setcounter{AlgoLine}{0}
\KwInput{direction $\vu$; starting point $\vw$.}
\KwOutput{step size $\mu$.}
\Indp
$\mu\leftarrow \frac{\langle \Phi \vw - \vv, \Phi \vu \rangle + \beta \langle \vw - \bm{1}, \vu \rangle }{\|\Phi \vu\|_2^2 + \beta \|\vu\|_2^2}$ \hfill {\small \textit{(optimal $\mu$)}}

\Indm 
\KwReturn{$\mu$}
 \caption{LineSearch($\vu, \vw$)}\label{alg:ls}
\Indp
\end{algorithm}
\DecMargin{1.5em}
\end{minipage}
\hfill
\begin{minipage}{0.48\textwidth}
\IncMargin{1.5em}
\begin{algorithm}[H]
\SetAlgoLined
\setcounter{AlgoLine}{0}
\Indm 
\KwInput{starting point $\vw$.}
\KwOutput{improved $\vw$.}
\Indp
$\vu \leftarrow [\nabla f_1(\vw)]_{\supp(\vw)}$ \hfill {\small \textit{(in-support grad.)}}

$\mu\leftarrow$ LineSearch($\vu, \vw$) 

$\vw\leftarrow \vw - \mu \vu$ \hfill {\small \textit{(in-support adjustment)}}

\Indm 
\KwReturn{$\vw$}
 \caption{De-bias($\vw$)}\label{alg:db}
\Indp
\end{algorithm}
\DecMargin{1.5em}
\end{minipage}

Recall the inner optimization (line~6) of Algorithm~\ref{alg:iht} is
\begin{align}
    \vw\leftarrow \underset{\vw \in \sR^{n_u}_+, \|\vw\|_0\leq b}{\argmin} \tfrac{1}{2}\|\bm{w}-\bm{s}\|_2^2 + f_2(\bm{w}).
\end{align}
Noting that $\tfrac{1}{2}\|\bm{w}-\bm{s}\|_2^2 + f_2(\bm{w}) = \textstyle\sum_{j\in [n_u]}(\tfrac{1}{2}(w_j-s_j)^2-\alpha \sigma_j^2)$, this step can be done optimally by simply picking the top-$b$ elements, as shown in the following. Given a $b$-sparse support set $\gS\subset [n_u]$, we can see that
$$
    \underset{\vw \in \sR^{n_u}_+, \supp(\vw)\subseteq \gS}{\min} \textstyle\sum_{j\in [n_u]}(\tfrac{1}{2}(w_j-s_j)^2-\alpha \sigma_j^2)=\textstyle\sum_{j\in \gS}(\tfrac{1}{2}[-s_j]_+^2-\alpha \sigma_j^2).
$$
Therefore, line~6 in Algorithm~\ref{alg:iht} can be done by: (1) find the $b$ smallest $(\tfrac{1}{2}[-s_j]_+^2-\alpha \sigma_j^2)$, denoting the resulting $b$-sparse index set as $\gS^\star$; (2) let $\vw\leftarrow [[\vs]_{\gS^\star}]_+$. 

\section{More Experiment Results}\label{sec:more_exp}
\subsection{Ablation Study: trade-off of uncertainty and representation}\label{sec:exp-ablation}
\begin{wrapfigure}{R}{0.38\textwidth}
  \centering
  \includegraphics[width=1\linewidth]{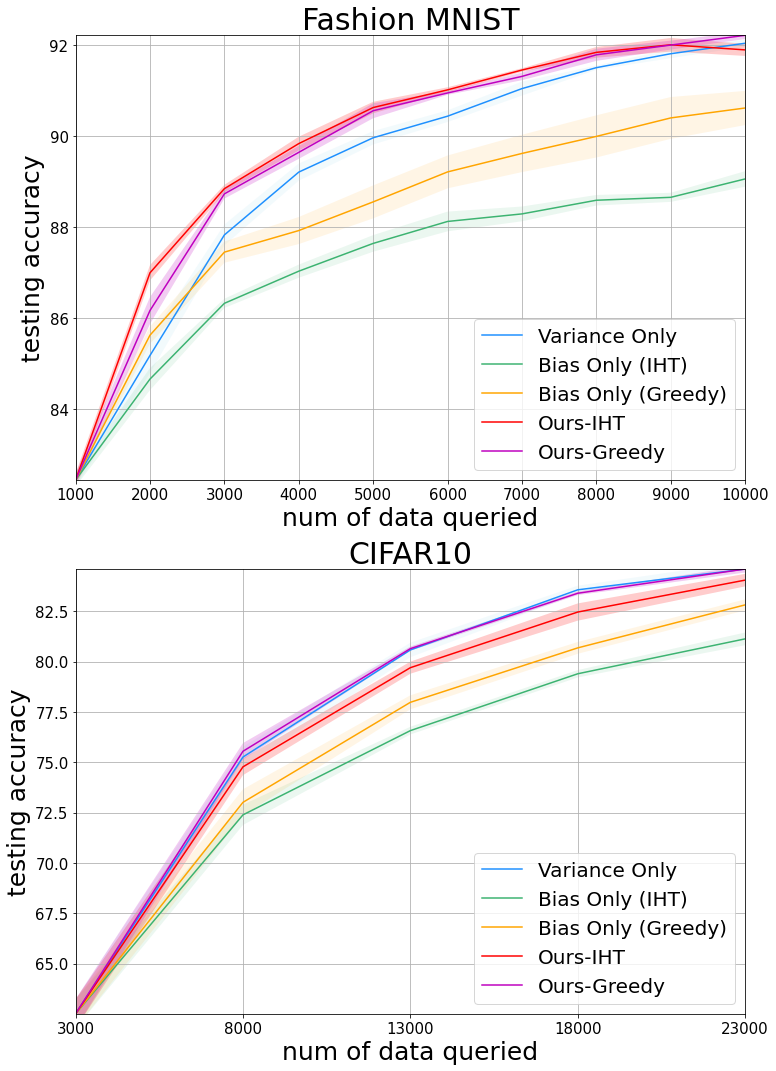}
  \caption{Ablation Study Results on Bayesian models.}
  \label{fig:ablation}
  \vspace{-4em}
\end{wrapfigure}
We perform an ablation study to understand better the trade-off between the variance and the bias terms in our final formulation~\eqref{eq:opt-1}. To remove the bias term, we query the data with top variances. To remove the variance term, we query the data by only minimizing the approximation bias, i.e., setting $\alpha=0$, under both IHT and Greedy optimizations respectively. We take two datasets MNIST and CIFAR-10 in the Bayesian experiment as examples. Results in Figure~\ref{fig:ablation} demonstrate the necessity of taking both uncertainty and representation into consideration during the data acquisitions for ideal performance, while for some datasets like CIFAR-10, the variance contributes much more significantly.

\section{Implementation Details}\label{sec:exp-2}
All experiments are written in PyTorch 1.8.1 and trained on a single NVIDIA Tesla V100 GPU. All hyper-parameters are chosen to ensure models achieve good and stable performance on each dataset, and they are kept identical for all active learning approaches. 
\subsection{Bayesian Active Learning Experiment}
\paragraph{Model Architecture}
We use the exact same model as~\citep{pinsler2019bayesian}, it is a Bayesian neural network consisting of a ResNet-18~\citep{he2016deep} feature extractor followed by a fully connected layer with a ReLU activation, and a final layer allows sampling by local reparametrization~\citep{kingma2015variational} with a softmax activation.

\paragraph{Training details}
Different scales and complexities of different datasets determine their training details. For example, on Fashion MNIST dataset, we use 100 samples for random projections, 1000 seed data, and query 1000 samples for 9 iterations. On CIFAR-10, we use 2000 samples for random projections, 3000 seed data, and query 5000 samples for 4 iterations. On CIFAR-100, we use 2000 samples for random projections, 10000 for seed data, and query 5000 samples for 4 iterations. Because Bayesian Coreset usually finds a much smaller batch than requested, for a fair comparison, we let Bayesian Coreset acquire more data than the batch size and stop the acquisition as long as it has selected a full batch of data. 

\paragraph{Optimization and Hyperparameter Selection}
Due to larger models and more complicated classification tasks, e.g., CIFAR-100, data augmentation(including random cropping and random horizontal flipping) and learning rate scheduler are used in this experiment to achieve good model performance. The model is optimized with the Adam~\citep{kingma2014adam} optimizer using default exponential decay rates (0.9, 0.999) for the moment estimates. Table~\ref{table: bayesian} shows the hyper-parameters in experiment on Bayesian batch active learning, where $bs$ denotes the batch size in dataloader during the model training, $lr$ denotes the learning rate, and $wd$ denotes the weight decay. The hyper-parameters are chosen through grid search.

\begin{table}[h!]
  \caption{Hyperparameters used in Bayesian active learning experiment\ }
  \centering
  \begin{tabular}{lllllllll}
    \textbf{Dataset} & \textbf{Method} & \textbf{Epoch} & $\boldsymbol{bs}$ & $\boldsymbol{\alpha}$ & $\boldsymbol{\beta}$ & $\boldsymbol{lr}$ & $\boldsymbol{wd}$\\
    \toprule
    Fashion MNIST & Ours-IHT & 200 & 256 & $1$ & $10^{-3}$ & 0.001 & $5\times10^{-4}$ \\
    Fashion MNIST & Ours-Greedy & 200 & 256 & $2$ & $0.5$ & 0.001 & $5\times10^{-4}$ \\
    CIFAR-10 & Ours-IHT & 200 & 256 & $1$ & $10^{-6}$ & 0.001 & $5\times10^{-4}$ \\
    CIFAR-10 & Ours-Greedy & 200 & 256 & $2$ & $1$ & 0.001 & $5\times10^{-4}$ \\
    CIFAR-100 & Ours-IHT & 200 & 256 & $1$ & $10^{-6}$ & 0.001 & $5\times10^{-4}$ \\
    CIFAR-100 & Ours-Greedy & 200 & 256 & $1$ & $0.5$ & 0.001 & $5\times10^{-4}$ \\
  \end{tabular}
  \label{table: bayesian}
\end{table}

\subsection{General Active Learning Experiment}
\paragraph{Model Architecture}
On MNIST dataset, we use LeNet-5 model~\citep{lecun2015lenet}. On SVHN and CIFAR10 datasets, we use VGG-16 model~\citep{simonyan2014very}.

\paragraph{Training details}
On MNIST dataset with LeNet-5 model, we use 40 seed data, and query 40 samples for 15 iterations. On SVHN with VGG-16 model, which contains more complicated real-world color images, we use 1000 seed data, and query 1000 samples for 5 iterations. On CIFAR-10 with VGG-16 model, we use 3000 seed data, and query 3000 samples for 5 iterations. 

\paragraph{Optimization and Hyperparameter Selection}
All models are trained using the cross entropy loss with SGD optimizer, and no data augmentation or learning rate scheduler is used. Tabel \ref{table: non-bayesian} shows the hyper-parameters in experiment on general batch active learning, where $bs$ denotes the batch size in dataloader during the model training, $lr$ denotes the learning rate, $m$ denotes the momentum, and $wd$ denotes the weight decay. The hyper-parameters are chosen through grid search.
\begin{table}[h!]
  \caption{Hyperparameters used in general active learning experiment\ }
  \centering
  \begin{tabular}{llllllllll}
    \textbf{Dataset} & \textbf{Method} & \textbf{Epoch} & $\boldsymbol{bs}$ & $\boldsymbol{\alpha}$ & $\boldsymbol{\beta}$ & $\boldsymbol{lr}$ & $\boldsymbol{wd}$\\
    \toprule
    MNIST & Ours-IHT & 150 & 32 & $10^{-8}$ & $10^{-4}$ & 0.01 & 0.9 & $5\times10^{-4}$ \\
    MNIST & Ours-Greedy & 150 & 32 & $10^{-8}$ & $10^{-1}$ & 0.01 & 0.9 & $5\times10^{-4}$ \\
    SVHN & Ours-IHT & 150 & 128 & $10^{-8}$ & $10^{-4}$ & 0.01 & 0.9 & $5\times10^{-4}$ \\
    SVHN & Ours-Greedy & 150 & 128 & $10^{-1}$ & $10^{-1}$ & 0.01 & 0.9 & $5\times10^{-4}$ \\
    CIFAR-10 & Ours-IHT & 100 & 128 & $10^{-8}$ & $10^{-6}$ & 0.001 & 0.9 & $5\times10^{-4}$ \\
    CIFAR-10 & Ours-Greedy & 100 & 128 & $10^{-2}$ & $10^{-1}$ & 0.001 & 0.9 & $5\times10^{-4}$ \\
  \end{tabular}
  \label{table: non-bayesian}
\end{table}

\end{document}